\documentclass[11pt]{article}

\usepackage{amsmath, amsthm, amssymb, amscd, paralist}
\usepackage{graphicx, color}
\usepackage{bm}
\usepackage[title]{appendix}
\usepackage{hyperref}
\newcommand{\eps}{\varepsilon}

\numberwithin{equation}{section}
\newtheorem{thm}{Theorem}
\newtheorem{lem}{Lemma}

\newcommand{\R}{\mathbb{R}}

\newcommand{\ba}{\mathbf{a}}

\newcommand{\mC}{\mathcal{C}}

\newcommand{\wt}{\widetilde}

\renewcommand{\ba}{\mathbf{a}}

\newcommand{\bw}{\mathbf{w}}
\newcommand{\bx}{\mathbf{x}}

\newcommand{\by}{\mathbf{y}}

\begin{document}

% "Title of the paper"
\title{The Kolmogorov-Arnold representation theorem revisited}

\author{Johannes Schmidt-Hieber\footnote{University of Twente and Leiden University \newline
{\small {\em Address:} Drienerlolaan 5, 7522 NB Enschede, The Netherlands} \newline {\small {\em Email:} \texttt{a.j.schmidt-hieber@utwente.nl}, \texttt{schmidthieberaj@math.leidenuniv.nl} \newline The research has been supported by the Dutch STAR network and a Vidi grant from the Dutch science organization (NWO). This work was done while the author was visiting the Simons Institute for the Theory of Computing. The constructive comments and suggestions shared by the associate editor and the three referees resulted in a significantly improved version of the article. The author wants to thank Matus Telgarsky for helpful remarks and pointing to the article \cite{MR1288945}.}}}

\date{}
\maketitle

\begin{abstract}
There is a longstanding debate whether the Kolmogorov-Arnold representation theorem can explain the use of more than one hidden layer in neural networks. The Kolmogorov-Arnold representation decomposes a multivariate function into an interior and an outer function and therefore has indeed a similar structure as a neural network with two hidden layers. But there are distinctive differences. One of the main obstacles is that the outer function depends on the represented function and can be wildly varying even if the represented function is smooth. We derive modifications of the Kolmogorov-Arnold representation that transfer smoothness properties of the represented function to the outer function and can be well approximated by ReLU networks. It appears that instead of two hidden layers, a more natural interpretation of the Kolmogorov-Arnold representation is that of a deep neural network where most of the layers are required to approximate the interior function. 
\end{abstract}

% \paragraph{AMS 2010 Subject Classification:}
% Primary 
%
% %62G10   	Hypothesis testing
% %62G15   	Tolerance and confidence regions
% %62G20   	Asymptotic properties
%
\paragraph{Keywords:} Kolmogorov-Arnold representation theorem; function approximation; deep ReLU networks; space-filling curves.

\newpage

\section{Introduction}

Why are additional hidden layers in a neural network helpful? The Kolmogorov-Arnold representation (KA representation in the following) seems to offer an answer to this question as it shows that every continuous function can be represented by a specific network with two hidden layers \cite{hechtnielsen1987}. But this interpretation has been highly disputed. Articles discussing the connection between both concepts have titles such as  "Representation properties of networks: Kolmogorov's theorem is irrelevant" \cite{girosipoggio1989} and "Kolmogorov's theorem is relevant" \cite{Kurkova1991}. 

The original version of the KA representation theorem states that for any continuous function $f: [0,1]^d \rightarrow \R,$ there exist univariate continuous functions $g_q,$ $\psi_{p,q}$ such that 
\begin{align}
	f(x_1, \ldots, x_d) = \sum_{q=0}^{2d} g_q\Big( \sum_{p=1}^d \psi_{p,q}(x_p)\Big).
	\label{eq.KA_rep1}
\end{align}
%Ivakhnenko and Lapa are apparently the first that proposed something with multiple layers, see survey articles by Schmidhuber
This means that the $(2d+1)(d+1)$ univariate functions $g_q$ and $\psi_{p,q}$ are enough for an exact representation of a $d$-variate function. Kolmogorov published the result in 1957 disproving the statement of Hilbert's 13th problem that is concerned with the solution of algebraic equations. The earliest proposals in the literature introducing multiple layers in neural networks date back to the sixties and the link between KA representation and multilayer neural networks occurred much later.

A ridge function is a function of the form $f(\bx)= \sum_{p=1}^m g_p(\bw_p^\top \bx ),$ with vectors $\bw_p \in \R^d$ and univariate functions $g_p$. The structure of the KA representation can therefore be viewed as the composition of two ridge functions. There exists no exact representation of continuous functions by ridge functions and matching upper and lower bounds for the best approximations are known \cite{MR1696577, MR1696573, MR1866380}. The composition structure is thus essential for the KA representation. A two-hidden-layer feedforward neural network with activation function $\sigma,$ hidden layers of width $m_1$ and $m_2,$ and one output unit can be written in the form
\begin{align*}
	f(\bx) =  \sum_{q=1}^{m_1} d_q \, \sigma\Big(\sum_{p=1}^{m_2} b_{pq}\sigma(\bw_p^\top \bx	+a_p)+c_q\Big), \  \text{with parameters} \ \bw_p \in \mathbb{R}^d, \ a_p, b_{pq}, c_q, d_q \in \mathbb{R}.
\end{align*}
Because of the similarity between the KA representation and neural networks, the argument above suggests that additional hidden layers can lead to unexpected features of neural network functions.

There are several reasons why the Kolmogorov-Arnold representation theorem has been initially declared as irrelevant for neural networks in \cite{girosipoggio1989}. The original proof of the KA representation in \cite{Kolmogorov1957} and some later versions are non-constructive providing very little insight on how the function representation works. Although the $\psi_{p,q}$ are continuous, they are still rough functions sharing similarities with the Cantor function. Meanwhile more refined KA representation theorems have been derived strengthening the connection to neural networks \cite{MR210852, SPRECHER1996765,SPRECHER1997447,MR2558696}. \cite{MAIOROV199981} showed that the KA representation can essentially be rewritten in the form of a two-hidden-layer neural network for a non-computable activation function $\sigma,$ see the literature review in Section \ref{sec.rel_lit} for more details. The following KA representation is much more explicit and practical.

\begin{thm}[Theorem 2.14 in \cite{braunthesis2009}]
Fix $d \geq 2.$ There are real numbers $a, b_p, c_q$ and a continuous and monotone function $\psi: \R \rightarrow \R,$ such that for any continuous function $f:[0,1]^d \rightarrow \R,$ there exists a continuous function $g: \R \rightarrow \R$ with
\begin{align*}
	f(x_1, \ldots, x_d) = \sum_{q=0}^{2d} g\Big(\sum_{p=1}^d b_p \psi(x_p + qa) + c_q\Big). 
\end{align*}
\end{thm}

This representation is based on translations of one inner function $\psi$ and one outer function $g.$ The inner function $\psi$ is independent of $f.$ The dependence on $q$ in the first layer comes through the shifts $qa.$ The right hand side can be realized by a neural network with two hidden layers. The first hidden layer has $d$ units and activation function $\psi$ and the second hidden layer consists of $2d+1$ units with activation function $g.$

For a given $0<\beta \leq 1,$ we will assume that the represented function $f$ is $\beta$-smooth, which here means that there exists a constant $C,$ such that $|f(\bx)-f(\by)|\leq C\|\bx-\by\|_\infty^\beta$ for all $\bx,\by \in [0,1]^d.$ Let $m>0$ be arbitrary. To approximate a $\beta$-smooth function up to an error $m^{-\beta}$, it is well-known that standard approximation schemes need at least of the order of $m^d$ parameters. This means that any efficient neural network construction mimicking the KA representation and approximating $\beta$-smooth functions up to error $m^{-\beta}$ should have at most of the order of $m^d$ many network parameters. 

Starting from the KA representation, the objective of the article is to derive a deep ReLU network construction that is optimal in terms of number of parameters. For that reason, we first present novel versions of the KA representation that are easy to prove and also allow to transfer smoothness from the multivariate function to the outer function. In Section \ref{sec.deepReLU} the link is made to deep ReLU networks. 

The efficiency of the approximating neural network is also the main difference to the related work \cite{KURKOVA1992501, MONTANELLI20201}. Based on sigmoidal activation functions, the proof of Theorem 2 in \cite{KURKOVA1992501} proposes a neural network construction based on the KA representation with two hidden layers and $dm(m+1)$ and $m^2(m+1)^d$ hidden units to achieve approximation error of the order of $m^{-\beta}.$ This means that more than $m^{4+d}$ network weights are necessary, which is sub-optimal in view of the argument above. The very recent work \cite{MONTANELLI20201} uses a modern version of the KA representation that guarantees some smoothness of the interior function. Combined with the general result on function approximation by deep ReLU networks in \cite{Yarotsky2018}, a rate is derived that depends on the smoothness of the outer function via the function class $K_C([0,1]^d;\R)$, see p.4 in \cite{MONTANELLI20201} for a definition. The non-trivial dependence of the outer function on the represented function $f$ makes it difficult to derive explicit expressions for the approximation rate if $f$ is $\beta$-smooth.  Moreover, as the KA representation only guarantees low regularity of the interior function, it remains unclear whether optimal approximation rates can be obtained. 

Although the deep ReLU network proposed in \cite{Shen_2020} is not motivated by the KA representation or space-filling curves, the network construction is quite similar. Section \ref{sec.rel_lit} contains a more detailed comparison with this and other related approaches.

\section{New versions of the KA representation}

The starting point of our work is the apparent connection between the KA representation and space-filling curves. A space-filling curve $\gamma$ is a surjective map $[0,1] \rightarrow [0,1]^d.$ This means that it hits every point in $[0,1]^d$ and thus "fills" the cube $[0,1]^d.$ Known constructions are based on iterative procedures producing fractal-type shapes. If $\gamma^{-1}$ exists, we could then rewrite any function $f:[0,1]^d\rightarrow \R$ in the form 
\begin{align}
	f= \underbrace{(f\circ \gamma)}_{=: g} \circ \gamma^{-1}.
	\label{eq.main_decompo}
\end{align}
This would decompose the function $f$ into a function $\gamma^{-1} :\R^d \rightarrow [0,1]$ that can be chosen to be independent of $f$ and a univariate function $g=f\circ \gamma : [0,1] \rightarrow \R$ containing all the information of the $d$-variate function $f.$ Compared to the KA representation, there are two differences. Firstly, the interior function $\gamma^{-1}$ is $d$-variate and not univariate. Secondly, by Netto's theorem \cite{KUPERS}, a continuous surjective map $[0,1]\rightarrow [0,1]^2$ cannot be injective and $\gamma^{-1}$ does not exist. The argument above can therefore not be made precise for arbitrary dimension $d$ and a continuous space-filling curve $\gamma.$

%Netto's theorem 
%https://pdfs.semanticscholar.org/a492/b12a02efaf01f1debfefd7cd41bb2ac66425.pdf

To illustrate our approach, we first derive a simple KA representation based on \eqref{eq.main_decompo} and with $\gamma^{-1}$ an additive function. The identity avoids the continuity of the functions $\psi$ and $g,$ which is the major technical obstacle in the proof of the KA representation. The proof does moreover not require that the represented function $f$ is continuous.

\begin{lem}
\label{lem.KA_lem}
Fix integers $d,B\geq 2.$ There exists a monotone function $\psi: [0,1] \rightarrow \R$ such that for any function $f:[0,1]^d \rightarrow \R,$ we can find a function $g: \R\rightarrow \R$ with 
\begin{align}
	f(x_1, \ldots, x_d) = g \Big( \sum_{p=1}^d B^{-p} \psi(x_p) \Big).
	\label{eq.Z_curve_represent}
\end{align}
\end{lem}

\begin{proof}
The B-adic representation of a number is not unique. For the decimal representation, $1$ is for instance the same as $1=0.999\ldots$ To avoid any problems that this may cause, we select for each real number $x\in [0,1]$ one B-adic representation $x=\sum_{j\geq 1} B^{-j}a_j^x$ with $a_j^x \in \{0, \ldots, B-1\}.$ Throughout the following, it is often convenient to rewrite $x$ in its B-adic expansion. Set
\begin{align*}
	x= \sum_{j=1}^{\infty}\frac{a_j^x}{B^j}=:[0. a_1^x a_2^x a_3^x\ldots]_B
\end{align*}
and define the function
\begin{align*}
	\psi(x)= \sum_{j=1}^ \infty \frac{a_j^x}{B^{d(j-1)}}.	
\end{align*}
The function $\psi$ is monotone and maps $x$ to a number with $B$-adic representation $$[a_1^x.\underbrace{0\ldots\ldots0}_{(d-1)\text{-times}}a_2^x\underbrace{0\ldots\ldots0}_{(d-1)\text{-times}}a_3^x0\ldots\ldots]_B$$ inserting always $d-1$ zeros between the original $B$-adic digits of $x.$ Multiplication by $B^{-p}$ shifts moreover the digits by $p$ places to the right. From that we obtain the $B$-adic representation
\begin{align}
	\Psi(x_1, \ldots,x_d):=\sum_{p=1}^d B^{-p} \psi(x_p) = \big[0.a_1^{x_1}a_1^{x_2}\ldots a_1^{x_d} a_2^{x_1} \ldots\big]_B
	\label{eq.KA_funct}
\end{align}
Because we can recover $x_1,\ldots, x_d$ from $\Psi(x_1, \ldots,x_d),$ the map $\Psi$ is invertible. Denote the inverse by $\Psi^{-1}.$ We can now define $g=f \circ \Psi^{-1}$ and this proves the result.
\end{proof}

The proof provides some insights regarding the structure of the KA representation. Although one might find the construction of $\Psi:[0,1]^d\rightarrow [0,1]$ in the proof very artificial, a substantial amount of neighborhood information persists under $\Psi.$ Indeed, points that are close are often mapped to nearby values. If for instance $\bx_1, \bx_2\in [0,1]^d$ are two points coinciding in all components up to the $k$-th $B$-adic digit, then, $\Psi(\bx_1)$ and $\Psi(\bx_2)$ coincide up to the $kd$-th $B$-adic digit. In this sense, the KA representation can be viewed as a two step procedure, where the first step $\Psi$ does some extreme dimension reduction. Compared to low-dimensional random embeddings which by the  Johnson-Lindenstrauss lemma nearly preserve the Euclidean distances among points, there seems, however, to be no good general characterization of how the interior function changes distances. 

The function $\Psi$ is discontinuous at all points with finite $B$-adic representation.  The map $\Psi$ defines moreover an order relation on $[0,1]^d$ via $\bx <\by :\Leftrightarrow \Psi(\bx) < \Psi(\by).$ For $B=d=2,$ the inverse map $\Psi^{-1}$ is often called the Morton order and coincides, up to a rotation of $90$ degrees, with the $z$-curve in the theory of space-filling curves (\cite{Bader2013}, Section 7.2). 

If $f$ is a piecewise constant function on a dyadic grid, the outer function $g$ is also piecewise constant. As a negative result, we show that for this representation, smoothness of $f$ does not translate into smoothness on $g.$

\begin{lem}
Let $k$ be a positive integer. Consider representation \eqref{eq.Z_curve_represent} for $B=2$ and let $g$ be as in the proof of Lemma \ref{lem.KA_lem}.
\begin{compactitem}
\item[(i)] If $f:\R^d \rightarrow \R$ is piecewise constant on the $2^{kd}$ hypercubes $\times_{j=1}^d (\ell_j 2^{-k}, (\ell_j +1)2^{-k}),$ with $\ell_1, \ldots, \ell_d \in \{0,2^k-1\},$ then $g$ is a piecewise constant function on the intervals $(\ell 2^{-kd}, (\ell+1)2^{-kd}),$ $\ell=0,\ldots,2^{kd}-1.$ 
\item[(ii)] If $f(x)=x,$ then $g$ is discontinuous. 
\end{compactitem}
\end{lem}

\begin{proof}
{\it (i):} If $x\in (\ell 2^{-kd}, (\ell+1)2^{-kd}),$ we can write $x=\Delta+\ell 2^{-kd}$ with $0<\Delta<2^{-kd}.$ There exist thus $\ell_1, \ldots, \ell_d \in \{0,2^k-1\},$ such that $\Psi^{-1}(x)=\Psi^{-1}(\Delta) +(\ell_1 2^{-k},\ldots, \ell_d 2^{-k}).$ Since $\Psi^{-1}(\Delta) \in (0,2^{-k})\times \ldots \times (0,2^{-k}),$ the result follows from $g=f\circ \Psi^{-1}.$ \\
{\it (ii):} If $f$ is the identity, $g=f\circ \Psi^{-1}=\Psi^{-1}.$ For $x\uparrow 1/2,$ we find that $\Psi^{-1}(x) \rightarrow (1/2,1,1,\ldots, 1)$ and  for $x \downarrow 1/2,$ $\Psi^{-1}(x) \rightarrow (1/2,0,0,\ldots,0).$ Even stronger, every point with finite binary representation is a point of discontinuity.
\end{proof}

The discontinuity of the space-filling map $\Psi^{-1}$ causes $g$ to be more irregular than $f.$ Many constructions of space-filling curves are known but to obtain a representation of KA type $\Psi$ needs to be an additive function. The additivity condition rules out most of the canonical choices, such as for instance the Hilbert curve. Below, we use for $\Psi^{-1}$ the Lebesgue curve and show that this then leads to a representation that allows to transfer smoothness properties of $f$ to smoothness properties on $g$ and therefore overcomes the shortcomings of the representation in \eqref{eq.Z_curve_represent}. In contrast to the earlier result, $g$ is now a function that maps from the Cantor set, in the following denoted by $\mC,$ to the real numbers.

\begin{thm}
\label{thm.KA2_lem}
For fixed dimension $d\geq 2,$ there exists a monotone function $\phi: [0,1] \rightarrow \mC$ (the Cantor set) such that for any function $f:[0,1]^d \rightarrow \R,$ we can find a function $g: \mC\rightarrow \R$ such that 
\begin{compactitem}
\item[(i)] 
\begin{align}
	f(x_1, \ldots, x_d) = g \Big( 3\sum_{p=1}^d 3^{-p} \phi(x_p) \Big);
	\label{eq.Lebesgue_curve_represent}
\end{align}
\item[(ii)] if $f:[0,1]^d \rightarrow \R$ is continuous, then also $g: \mC\rightarrow \R$ is continuous;
\item[(iii)] if there exists $\beta \leq 1$ and a constant $Q,$ such that $|f(\bx)-f(\by)|\leq Q|\bx-\by|_\infty^\beta,$ for all $\bx, \by \in [0,1]^d,$ then, 
\begin{align*}
	|g(x)-g(y)|\leq 2^\beta Q |x-y|^{\frac{\beta\log 2}{d\log 3}}, \quad \text{for all} \ x,y \in \mC;
\end{align*}
\item[(iv)] if $f(\bx)=\bx,$ then, there exist sequences $(x_k)_k, (y_k)_k \subset \mC$ with $\lim_k x_k =\lim_k y_k$ and $$\big|g(x_k)-g(y_k)\big|_\infty=\Big (\frac{|x_k-y_k|}2\Big)^{\frac{\log 2}{d\log 3}}.$$

%If $f$ has H\"older index $\beta\leq 1,$ then $g: \mC\rightarrow \R$ has H\"older index $\beta \log 2/(d \log 3).$ 
\end{compactitem}
\end{thm}

\begin{proof} The construction of the interior function is similar as in the proof of Lemma \ref{lem.KA_lem}. We associate with each $x\in [0,1]$ one binary representation $x=[0.a_1^xa_2^x \ldots ]_2$ and define  
\begin{align}
	\phi(x):=  \sum_{j=1}^ \infty \frac{2a_j^x}{3^{1+d(j-1)}}
	= \big[0.(2a_1^x)\underbrace{0\ldots\ldots0}_{(d-1)\text{-times}}(2a_2^x)\underbrace{0\ldots\ldots0}_{(d-1)\text{-times}} \big]_3.
	\label{eq.phi_def}
\end{align}
The function $\phi$ multiplies the binary digits by two (thus, only the values $0$ and $2$ are possible) and then expresses the digits in a ternary expansion adding $d-1$ zeros between each two digits. By construction, the Cantor set consists of all $y \in [0,1]$ that only have $0$ and $2$ as digits in the ternary expansion. This shows that $\phi:[0,1]\to \mathcal{C}.$ Define now
\begin{align}
	\Phi(x_1, \ldots,x_d):=3\sum_{p=1}^d 3^{-p} \phi(x_p) = \big[0.(2a_1^{x_1})(2a_1^{x_2})\ldots (2a_1^{x_d})(2a_2^{x_1}) \ldots \big]_3
	\label{eq.Phi_def}
\end{align}
where the right hand side is written in the ternary system. 
Because we can recover the binary representation of $x_1,\ldots, x_d,$ the map $\Phi$ is invertible. Since $2a_\ell^{x_r} \in \{0,2\}$ for all $\ell \geq 1$ and $r\in \{1,\ldots, d\},$ the image of $\Phi$ is contained in the Cantor set. We can now define the inverse by $\Phi^{-1}: \mC \rightarrow \R$ and set $g=f \circ \Phi^{-1}: \mC \rightarrow \R,$ proving $(i).$

In a next step of the proof, we show that 
\begin{align}
	\big |\Phi^{-1}(x)- \Phi^{-1}(y)\big |_\infty \leq  2 |x-y|^{\log 2/(d\log 3)}, \quad \text{for all} \ x,y \in \mC,
	\label{eq.Phi-1_Hoeld}
\end{align}
For that we extend the proof in \cite{Bader2013}, p.98. Observe that $\Phi^{-1}$ maps $x=[0.x_1x_2x_3\ldots ]_3$ to the vector $([0.(x_1/2)(x_{d+1}/2)\ldots ]_2, \ldots, [0.(x_d/2)(x_{2d}/2)\ldots]_2)^\top \in [0,1]^d.$ Given arbitrary $x,y\in \mC,$ $k^*=k^*(x,y)$ denotes the integer $k$ for which $3^{-(k+1)d}\leq |x-y| < 3^{-kd}.$ Suppose that the first $k^*d$ ternary digits of $x$ and $y$ are not all the same and denote by $J$ the position of the first digit of $x$ that is not the same as $y.$ Since only the digits $0$ and $2$ are possible, the difference between $x$ and $y$ can be lower bounded by $|x-y|\geq 2\cdot 3^{-J}-3^{-J},$ where the term $-3^{-J}$ accounts for the effect of the later digits. Thus $|x-y|\geq 3^{-J}$ and this is a contradiction with $|x-y|<3^{-k^*d}$ and $J\leq k^*d.$ Thus, the first $k^*d$ ternary digits of $x$ and $y$ coincide. Using the explicit form of $\Phi^{-1}$ this also implies that $\Phi^{-1}(x)$ and $\Phi^{-1}(y)$ coincide in the first $k^*$ binary digits in each component. This means that $|\Phi^{-1}(x)-\Phi^{-1}(y)|_\infty \leq 2^{-k^*}$ and together with the definition of $k^*,$ we find
\begin{align}
	\big|\Phi^{-1}(x)-\Phi^{-1}(y)\big|_\infty 
	\leq 2\cdot 2^{-(k^*+1)}
	= 2 \big(3^{-(k^*+1)d}\big)^{\frac{\log 2}{d\log 3}}
	\leq 2|x-y|^{\frac{\log 2}{d\log 3}}
	\label{eq.Phi}
\end{align}
proving \eqref{eq.Phi-1_Hoeld}, since $x,y \in \mC$ were arbitrary. Using again that $g=f\circ \Phi^{-1},$ $(ii)$ and $(iii)$ follow.

To prove $(iv),$ take $x_k=0$ and $y_k=2/3^{kd}.$ Then, $\Phi^{-1}(x_k)=(0, \ldots, 0)^\top$ and $\Phi^{-1}(y_k)=(0,\ldots, 0, 2^{-k})^\top.$ Rewriting this yields $|g(x_k)-g(y_k)|_\infty=|\Phi^{-1}(x_k)-\Phi^{-1}(y_k)|_\infty= (|x_k-y_k|/2)^{\frac{\log 2}{d\log 3}}$ for all $k\geq 1.$
\end{proof}

Thus, by restricting to the Cantor set, one can overcome the limitations of Netto's theorem mentioned at the beginning of the section. In fact by construction and \eqref{eq.Phi}, $\gamma=\Phi^{-1}$ is a surjective, invertible and continuous space-filling curve. The previous theorem is in a sense more extreme than the KA representation as the univariate interior function maps to a set of Hausdorff dimension $\log 2/\log 3 <1.$ \cite{MR1288945} uses a similar construction to prove embeddings of the function spaces generated by circuits into neural network function classes. 

Representation \eqref{eq.Lebesgue_curve_represent} has the advantage that smoothness imposed on $f$ translates into smoothness properties on $g.$  The reason is that the function $\Phi$ associates to each $\bx\in [0,1]^d$ one value in the Cantor set such that values that are far from each other in $[0,1]^d$ are not mapped to nearby values in the Cantor set. Based on this value in the Cantor set, the outer function $g$ reconstructs the function value $f(\bx).$ Since the distance of the values in $[0,1]^d$ is linked to the distance of the values in the Cantor set, local variability in the function $f$ does not lead to arbitrarily large fluctuations of the outer function $g.$ Therefore smoothness imposed on $f$ translates into smoothness properties on $g.$

A natural question is whether we gain or loose something if instead of approximating $f$ directly, we use \eqref{eq.Lebesgue_curve_represent} and approximate $g.$ Recall that the approximation rate should be $m^{-\beta}$ if $m^d$ is the number of free parameters of the approximating function, $\beta$ the smoothness and $d$ the dimension. Since $g$ is by $(iii)$ $\alpha$-smooth with $\alpha=\beta\log 2/(d\log 3)$ and is defined on a set with Hausdorff dimension $d^*=\log 2/\log 3,$ we see that there is no loss in terms of approximation rates since $\beta/d=\alpha/d^*.$ Thus, we can reduce multivariate function approximation to univariate function approximation on the Cantor set. This, however, only holds for $\beta \leq 1.$ Indeed, the last statement of the previous theorem means that for the smooth function $f(x)=x,$ the outer function $g$ is not more than $\beta\log 2/(d\log 3)$-smooth, implying that for higher order smoothness, there seems to be a discrepancy between the multivariate and univariate function approximation.

The only direct drawback of \eqref{eq.Lebesgue_curve_represent} compared to the traditional KA representation is that the interior function $\phi$ is discontinuous. We will see in Section \ref{sec.deepReLU} that $\phi$ can, however, be well approximated by a deep neural network.

It is also of interest to study the function class containing all $f$ that are generated by the representation in \eqref{eq.Lebesgue_curve_represent} for $\beta$-smooth outer function $g.$ Observe that if $g(x)=x,$ then $f$ coincides with the interior function which is discontinuous. This shows that for $\beta \leq 1,$ the class of all $f$ of the form \eqref{eq.Lebesgue_curve_represent} with $g$ a $\beta\log 2/(d\log 3)$-smooth function on the Cantor set $\mC$ is strictly larger than the class of $\beta$-smooth functions. Interestingly, the function class with Lipschitz continuous outer function $g$ contains all functions that are piecewise constant on a dyadic partition of $[0,1]^d.$

\begin{lem}
Consider representation \eqref{eq.Lebesgue_curve_represent} and let $k$ be a positive integer. If $f:[0,1]^d \rightarrow \R$ is piecewise constant on the $2^{kd}$ hypercubes $\times_{j=1}^d [\ell_j 2^{-k}, (\ell_j +1)2^{-k}),$ with $\ell_1, \ldots, \ell_d \in \{0,2^k-1\},$ then $g$ is a Lipschitz function with Lipschitz constant bounded by $2\|f\|_\infty 3^{kd}.$
\end{lem}

\begin{proof}
Let $\phi$ and $\Phi$ be the same as in the proof of Theorem \ref{thm.KA2_lem}. For any vector $\ba=(a_1, \ldots, a_{kd})\in \{0,2\}^{kd}$ define $I(\ba)=\{[0.a_1\ldots a_{kd}b_1 b_2 \ldots ]_3:b_1,b_2,\ldots \in \{0,2\}\}.$ There exist integers $\ell_1, \ldots, \ell_d \in \{0,\ldots,2^k-1\}$ such that $\Phi^{-1}(I(\ba)) \subseteq \times_{j=1}^d [\ell_j 2^{-k}, (\ell_j +1)2^{-k}).$ Since $f$ is constant on these dyadic hypercubes, $g(I(\ba))=(f\circ \Phi^{-1})(I(\ba))=$ const. If $\ba, \widetilde \ba \in \{0,2\}^{kd}$ and $\ba \neq \widetilde \ba,$ then, arguing as in the proof of Theorem \ref{thm.KA2_lem}, we find that $|x-y|\geq 3^{-kd}$ whenever $x\in I(\ba)$ and $y \in I(\widetilde \ba).$ Therefore, we have $|g(x)-g(y)|=0$ if $x,y \in I(\ba)$ and $|g(x)-g(y)|\leq 2\|g\|_\infty\leq 2\|f\|_\infty \leq 2\|f\|_\infty 3^{kd}|x-y|$ if $x\in I(\ba)$ and $y\in I(\widetilde \ba).$ Since $\ba, \widetilde \ba$ were arbitrary, the result follows.
\end{proof}

%We have to be more careful with the choice of the binary representation as for the previous results. Whenever $x$ has a finite binary representation it will be selected for the map $\phi.$ Under $\Phi,$ the hypercube $I(\ell_1,\ldots, \ell_d):=\times_{j=1}^d [\ell_j 2^{-k}, (\ell_j+1)2^{-k})$ is mapped to a set of numbers where the first $kd$ ternary digits are the same. For two different vectors $(\ell_1, \ldots, \ell_d) \neq (\widetilde \ell_1, \ldots, \widetilde \ell_d),$ $\Phi$ maps this to sets where the first $kd$ ternary digits are not all the same. By arguing as in the proof of Theorem \ref{thm.KA2_lem}, we can now conclude that $|\Phi(x)-\Phi(y)|\geq 3^{-kd},$ whenever $x\in I(\ell_1,\ldots, \ell_d)$ and $y \in I(\widetilde \ell_1,\ldots,\widetilde \ell_d).$

It is important to realize that the space-filling curves and fractal shapes occur because of the exact identity. It is natural to wonder whether the KA representation leads to an interesting approximation theory. For that, one wants to truncate the number of digits in \eqref{eq.phi_def}, hence reducing the complexity of the interior function. We obtain an approximation bound that only depends on the dimension $d$ through the smoothness of the outer function $g$.

%Because for the KA representation in Theorem \ref{thm.KA2_lem}, smoothness imposed on $f$ induces smoothness of the outer function $g,$ we obtain an approximation bound that is even independent of the dimension $d.$ {\JSH The curse of dimensionality is completely absorbed by the decay of smoothness }

\begin{lem}\label{lem.f_approx}
Let $d\geq 2$ and suppose $K$ is a positive integer. For $x=[0.a_1^xa_2^x \ldots ]_2,$ define $\phi_K(x):= \sum_{j=1}^K 2a_j^x 3^{-1-d(j-1)}.$ If there exists $\beta \leq 1$ and a constant $Q,$ such that $|f(\bx)-f(\by)|\leq Q|\bx-\by|_\infty^\beta,$ for all $\bx, \by \in [0,1]^d,$ then, we can find a univariate function $g,$ such that $|g(x)-g(y)|\leq 2^\beta Q |x-y|^{\frac{\beta\log 2}{d\log 3}},$ for all $x,y \in \mC,$ and 
\begin{align*}
	\Big| f(\bx) - g\Big( 3\sum_{p=1}^d 3^{-p} \phi_K(x_p) \Big) \Big| \leq  2Q 2^{-\beta K}, \quad \text{for all} \ \bx=(x_1,\ldots, x_d)^\top \in [0,1]^d.
\end{align*}
Moreover, $\|f\|_{L^\infty([0,1]^d)}=\|g\|_{L^\infty(\mC)}.$
\end{lem}

\begin{proof}
From the geometric sum formula, $\sum_{q=0}^\infty 3^{-q}=3/2.$ Let $\phi$ and $g$ be as in Theorem \ref{thm.KA2_lem} and $\phi_K$ as defined in the statement of the lemma. Since $\beta \leq 1$, we then have that $|g(x)-g(y)|\leq 2Q |x-y|^{\frac{\beta\log 2}{d\log 3}},$ for all $x,y \in \mC.$ Moreover, \eqref{eq.Phi_def} shows that $3\sum_{p=1}^d 3^{-p} \phi(x_p)$ and $3\sum_{p=1}^d 3^{-p} \phi_K(x_p)$ are both in the Cantor set $\mC$ and have the same first $Kd$ ternary digits. Thus, using  \eqref{eq.Lebesgue_curve_represent}, we find
\begin{align*}
	\Big| f(\bx) - g\Big( 3\sum_{p=1}^d 3^{-p} \phi_K(x_p) \Big) \Big|
	&= 
	\Big| g\Big( 3\sum_{p=1}^d 3^{-p} \phi(x_p) \Big) - g\Big( 3\sum_{p=1}^d 3^{-p} \phi_K(x_p) \Big) \Big| \\
	&\leq 
	2 Q \Big |3\sum_{p=1}^d 3^{-p} \big(\phi(x_p)-\phi_K(x_p)\big) \Big|^{\frac{\beta\log 2}{d\log 3}} \\
	&\leq 2 Q \Big |2\sum_{q=Kd+1}^\infty 3^{-q} \Big|^{\frac{\beta\log 2}{d\log 3}} \\
	&\leq 2 Q \Big |2 \cdot 3^{-dK-1} \sum_{q=0}^\infty 3^{-q} \Big|^{\frac{\beta\log 2}{d\log 3}} \\
	&\leq 2 Q 3^{-K\frac{\beta\log 2}{\log 3}}.
\end{align*}
Finally, $\|f\|_{L^\infty([0,1]^d)}=\|g\|_{L^\infty(\mC)}$ follows as an immediate consequence of the function representation  \eqref{eq.Lebesgue_curve_represent}.
\end{proof}

\begin{figure}[ht]
\begin{center}
	\includegraphics[scale=0.5]{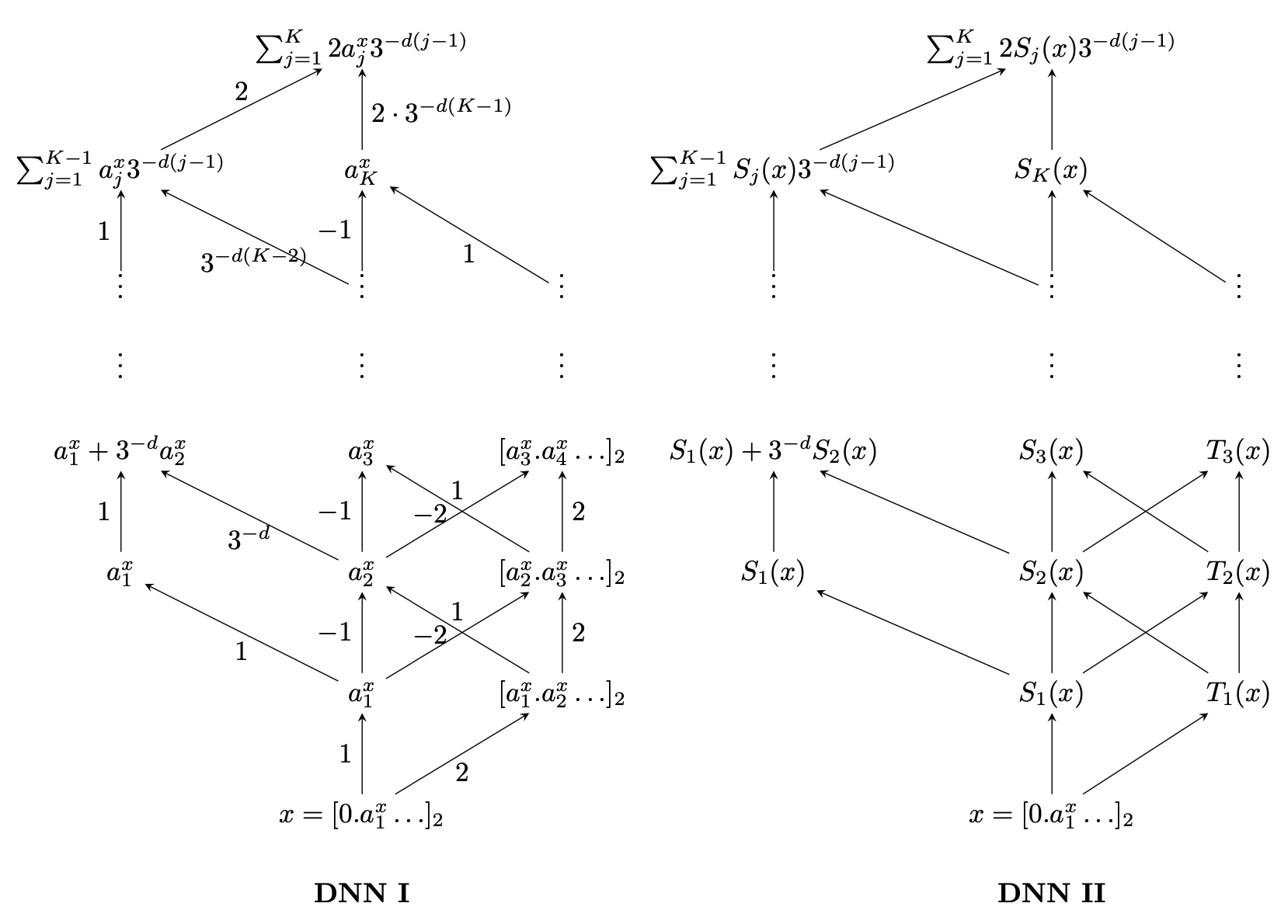} 
	\caption{\label{fig.bit_decoding} (Left) A deep neural network with $K$ hidden layers and width three computing the function $x=[0.a_1^xa_2^x\ldots ]_2 \mapsto 3\phi_K(x)=\sum_{j=1}^K 2a_j^x3^{-d(j-1)}$ exactly. In each hidden layer the linear activation function is applied to the left and right unit. The units in the middle use the threshold activation function $\sigma(x)=\mathbf{1}(x\geq 1/2).$ (Right) A deep ReLU network approximating the function $\phi_K.$ For the definitions of $S_r$ and $T_r$ see the proof of Theorem \ref{thm.main}.}
\end{center}
\end{figure}

\section{Deep ReLU networks and the KA representation}\label{sec.deepReLU}

This section studies the construction of deep ReLU networks imitating the KA approximation in Lemma \ref{lem.f_approx}. A deep/multilayer feedforward neural network is a function $\bx \mapsto f(\bx)$ that can be represented by an acyclic graph with vertices arranged in a finite number of layers. The first layer is called the input layer, the last layer is the output layer and the layers in between are called hidden layers. We say that a deep network has architecture $(L, (p_0,\dots,p_{L+1})),$ if the number of hidden layers is $L,$ and $p_0$, $p_j$ and $p_{L+1}$ are the number of vertices in the input layer, $j$-th hidden layer and output layer, respectively. The input layer of vertices represents the input $\bx.$ For all other layers, each vertex stands for an operation of the form $\by \mapsto \sigma(\ba^\top \by +b)$ with $\by$ the output (viewed as vector) of the previous layer, $\ba$ a weight vector, $b$ a shift parameter and $\sigma$ the activation function. Each vertex has its own set of parameters $(\ba,b)$ and also the activation function does not need to be the same for all vertices.  If for all vertices in the hidden layers the ReLU activation function $\sigma(x)=\max(x,0)$ is used and $L>1,$ the network is called a deep ReLU network. As common for regression problems, the activation function in the output layer will be the identity.

Approximation properties of deep neural networks for composed functions are studied in \cite{horowitz2007, Poggio2017, mhaskar2016, kohler2017, SH2017, bauer2019, 2019arXiv190702177N,2019arXiv190800695S, 2019arXiv190100137F}. These approaches do, however, not lead to straightforward constructions of ReLU networks exploiting the specific structure of the KA approximation
\begin{align}\label{eq.approx_identity_recall}
	f(\bx) \approx g\Big( 3\sum_{p=1}^d 3^{-p} \phi_K(x_p) \Big).
\end{align}
in Lemma \ref{lem.f_approx}. To find such a construction, recall that the classical neural network interpretation of the KA representation associates the interior function with the activation function in the first layer \cite{hechtnielsen1987}. Here, we argue that the interior function can be efficiently approximated by a deep ReLU network. The role of the hidden layers is to retrieve the next bit in the binary representation of the input.  Interestingly, some of the proposed network constructions to approximate $\beta$-smooth functions use a similar idea without making the link to the KA representation, see Section \ref{sec.rel_lit} for more details.

Figure \ref{fig.bit_decoding} gives the construction of a network computing $x=[0.a_1^xa_2^x\ldots ]_2 \mapsto 3\phi_K(x)=\sum_{j=1}^K 2a_j^x 3^{-d(j-1)}$ combining units with linear activation function $\sigma(x)=x$ and threshold activation function $\sigma(x) = \mathbf{1}(x \geq 1/2).$ The main idea is that for $x=[0.a_1^xa_2^x\ldots ]_2,$ we can extract the first bit using $a_1^x=\mathbf{1}(x\geq 1/2) = \sigma(x)$ and then define $2x-2\sigma(x)=2(x-a_1^x)=[0.a_2^xa_3^x \ldots ]_2.$ Iterating the procedure allows us to extract $a_2^x$ and consequently any further binary digit of $x.$ The deep neural network DNN I in Figure \ref{fig.bit_decoding} has $K$ hidden layers and network width three. The left units in the hidden layer successively build the output value; the units in the middle extract the next bit in the binary representation and the units on the right compute the remainder of the input after bit extraction. To learn the bit extraction algorithm, deep networks lead obviously to much more efficient representations compared to shallow networks. 

Constructing $d$ networks computing $\phi_K(x_p)$ for each $x_1, \ldots,x_d$ and combining them yields a network with $K+1$ hidden layers and network width $3d,$ computing the interior function $(x_1, \ldots, x_d) \mapsto 3\sum_{p=1}^d 3^{-p} \phi_K(x_p)$ in \eqref{eq.approx_identity_recall}. The overall number of non-zero parameters is of the order $Kd.$ To approximate a $\beta$-smooth function $f$ by a neural network via the KA approximation \eqref{eq.approx_identity_recall}, the interior step makes the approximating network deep but uses only very few parameters compared to the approximation of the univariate function $g.$ 

A close inspection of the network DNN I in Figure \ref{fig.bit_decoding} shows that all linear activation functions get non-negative input and can therefore be replaced by the ReLU activation function without changing the outcome. The threshold activation functions $\sigma(x) = \mathbf{1}(x \geq 1/2)$ can be arbitrarily well approximated by the linear combination of two ReLU units via $\eps^{-1}(x-(1-\eps)/2)_+ - \eps^{-1}(x-(1+\eps)/2)_+\approx \mathbf{1}(x\geq 1/2)$ for $\eps \downarrow 0.$ If one accepts potentially huge network parameters, the network DNN I in Figure \ref{fig.bit_decoding} can therefore be approximated by a deep ReLU network with $K$ hidden layers and network width four. Consequently, also the construction in \eqref{eq.approx_identity_recall} can be arbitrarily well approximated by deep ReLU networks. It is moreover possible to reduce the size of the network parameters by inserting additional hidden layers in the neural network, see for instance Proposition A.3 in \cite{2019arXiv190102220E}.

Throughout the following we write $\|f\|_p:=\|f\|_{L^p([0,1]^d)}.$ 

\begin{thm}\label{thm.main}
Let $p\in [1,\infty).$ If there exists $\beta \leq 1$ and a constant $Q,$ such that $|f(\bx)-f(\by)|\leq Q|\bx-\by|_\infty^\beta,$ for all $\bx, \by \in [0,1]^d,$ then, there exists a deep ReLU network $\wt f$ with $2K+3$ hidden layers, network architecture $(2K+3,(d,4d,\ldots,4d,d,1,2^{Kd}+1,1))$ and all network weights bounded in absolute value by $2(Kd\vee \|f\|_\infty) 2^{K(d \vee (p\beta))},$ such that
\begin{align*}
	\big\| f - \wt f \, \big\|_p \leq  2\big(Q+ \|f\|_\infty\big) 2^{-\beta K}.
\end{align*}
\end{thm}

\begin{proof}
The proof consists of four parts. In part (A) we construct a ReLU network mimicking the approximand constructed in Lemma \ref{lem.f_approx}. For that we first build a ReLU network with architecture $(2K,(1,4,\ldots,4,1))$ imitating the function $x=[0.a_1^xa_2^x\ldots ]_2 \mapsto 3\phi_K(x)=\sum_{j=1}^K  2a_j^x 3^{-d(j-1)}.$ In part (B), it is shown that the ReLU network approximation coincides with the function $3\phi_K$ on a subset of $[0,1]^d$ with Lebesgue measure $\geq 1- 2^{-K\beta p}.$ In part (C), we construct a neural network approximation for the outer function $g$ in Lemma \ref{lem.f_approx}. The approximation error is controlled in Part (D).

{\it (A):} Let $r$ be the largest integer such that $2^r\leq 2Kd 2^{K \beta p}$ and set $S_1(x):=2^r (x-1/2 +2^{-r-1})_+-2^r (x-1/2 -2^{-r-1})_+$ and $T_1(x):=2x.$ Given $S_j(x), T_j(x),$ we can then define 
\begin{align}
	T_{j+1}(x):=\big(2T_j(x)-2S_j(x)\big)_+, \quad 
	S_{j+1}(x):=S_1\big( T_j(x) -S_j(x)\big).
	\label{eq.TS_update}
\end{align}
There exists a ReLU network with architecture $(1,(1,2,1))$ and all network weights bounded in absolute value by $2^r$ computing the function $x\mapsto S_1(x).$ Similarly, there exists a ReLU network with architecture $(1,(2,2,1))$ computing $(S_j(x),T_j(x)) \mapsto S_{j+1}(x)=S_1( T_j(x) -S_j(x)).$ Since $S_1(x) \geq 0,$ we have that $(S_j(x))_+=S_j(x)$ and $T_j(x)=(T_j(x))_+.$ Because of that, we can now concatenate these networks as illustrated in Figure \ref{fig.bit_decoding} to construct a deep ReLU network computing $x\mapsto \sum_{j=1}^K 2S_j(x) 3^{-d(j-1)}.$ Recall that computing $S_{j+1}(x)$ from $(S_j(x),T_j(x))$ requires an extra layer with two nodes that is not shown in Figure \ref{fig.bit_decoding}. Thus, any arrow, except for the ones pointing to the output, adds one additional hidden layer to the ReLU network. The overall number of hidden layers is thus $2K.$ Because of the two additional nodes in the non-displayed hidden layers, the width in all hidden layers is four and thus the overall architecture of this deep ReLU network is $(2K,(1,4,\dots,4,1)).$ By checking all edges, it can be seen that all network weights are bounded by $2^r\leq 2Kd 2^{K\beta p}.$

{\it (B):} Recall that $x=[0.a_1^xa_2^x\dots]_2.$ We now show that on a large subset of the unit interval, it holds that $S_j(x)=a_j^x$ and $T_j(x)=[a_j^x.a_{j+1}^xa_{j+2}^x\dots]_2$ for all $j=1,\dots, K$ and therefore also $\sum_{j=1}^K 2S_j(x) 3^{-d(j-1)}=\sum_{j=1}^K 2a_j^x 3^{-d(j-1)}=3\phi_K(x).$

We have that $S_1(x)=\mathbf{1}(x>1/2),$ whenever $|x-1/2|\geq 2^{-r-1}.$ Set $A_{j,r}:=\{x :|[0.a_j^xa_{j+1}^x \dots]_2 -1/2|\geq 2^{-r-1}\}.$ If $x\in A_{j,r},$ then, $S_1([0.a_j^xa_{j+1}^x\dots]_2)=a_j^x.$ Thus, if $S_j(x)=a_j^x,$ $T_j(x)=[a_j^x.a_{j+1}^xa_{j+2}^x\dots]_2,$ and $x\in A_{j+1,r},$ then, \eqref{eq.TS_update} implies $S_{j+1}(x)=a_{j+1}^x$ and $T_{j+1}(x)=[a_{j+1}^x.a_{j+2}^xa_{j+3}^x\dots]_2.$ Hence, the deep ReLU network constructed in part (A) computes the function $[0,1] \ni x\mapsto 3\phi_K(x)$ exactly on the set $\cap_{j=1}^K A_{j,r}.$

For fixed $a_1^x,\dots,a_{j-1}^x \in \{0,1\},$ the set $\{x:|[0.a_j^xa_{j+1}^x \dots]_2 -1/2|< 2^{-r-1}\}$ is an interval of length $2^{-r-j+1}.$ As there are $2^{j-1}$ possibilities to choose $a_1^x,\dots,a_{j-1}^x \in \{0,1\},$ the complement $A_{j,r}^c=\{x\in [0,1]: x\notin A_{j,r}\}$ can be written as the union of $2^{j-1}$ subintervals of length $2^{-r-j+1}.$ The Lebesgue measure of $A_{j,r}^c$ is therefore bounded by $2^{-r}.$  Since $Kd 2^{K\beta p}\leq 2^r,$ we find that $(\cap_{j=1}^K A_{j,r})^c$ has Lebesgue measure bounded by $K2^{-r}\leq 2^{-K \beta  p}/d.$ This completes the proof for part (B).

\begin{figure}[ht]
\begin{center}
	\includegraphics[scale=0.5]{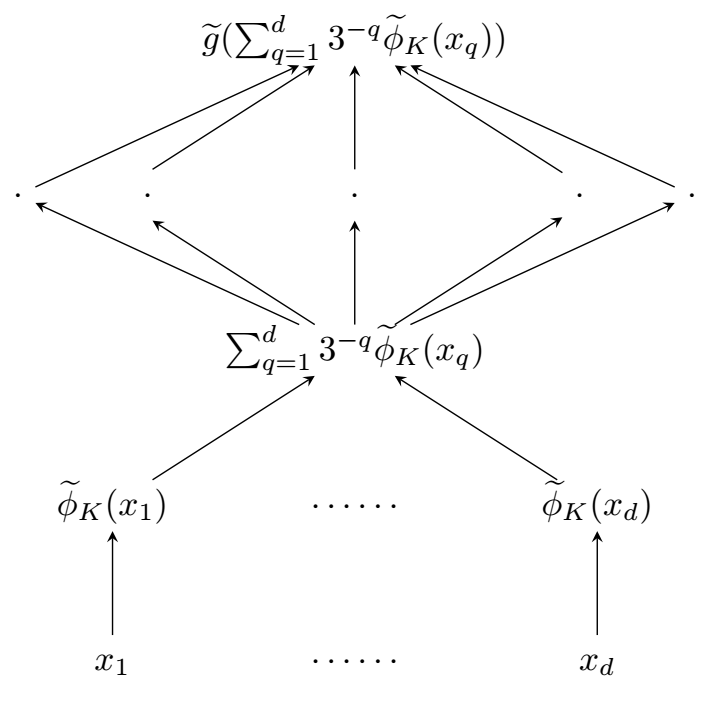} 
	\caption{\label{fig.NWcomp} Construction of the deep ReLU network in part (D) of the proof for Theorem \ref{thm.main}.}
\end{center}
\end{figure}

{\it (C):} Now we construct a shallow ReLU network interpolating the outer function $g$ in Lemma \ref{lem.f_approx} at the $2^{Kd}+1$ points $\{\sum_{j=1}^{Kd}  2t_j 3^{-j}: (t_1,\dots,t_{Kd}) \in \{0,1\}^{Kd}\} \cup \{1\}.$ Denote these points by $0=:s_0 < s_1< \dots < s_{2^{Kd}-1}<s_{2^{Kd}}:=1.$ For any $x\in [0,1],$ 
\begin{align*}
	\wt g(x) &:= g(s_0)+\sum_{j=1}^{2^{Kd}} \frac{g(s_j)-g(s_{j-1})}{s_j-s_{j-1}} \big((x-s_{j-1})_+-(x-s_j)_+\big) \\
	&=
	g(s_0)(x+1)_++\Big(\frac{g(s_1)-g(s_0)}{s_1-s_0}-g(s_0)\Big)(x)_+ \\
	&\quad +\sum_{j=1}^{2^{Kd}-1}\Big(
	\frac{g(s_{j+1})-g(s_j)}{s_{j+1}-s_j}-\frac{g(s_j)-g(s_{j-1})}{s_j-s_{j-1}}\Big)(x-s_j)_+.
\end{align*}
The function $\wt g(x)$ can therefore be represented on $[0,1]$ by a shallow ReLU network with $2^{Kd}+1$ units in the hidden layer. Moreover, $\wt g(s_j)=g(s_j)$ for all $j=0,\dots, 2^{Kd}.$ Finally, we bound the size of the network weights. We have $s_{j+1}-s_j \geq 3^{-Kd}.$ By Lemma \ref{lem.f_approx}, $\|f\|_\infty=\|g\|_{L^\infty(\mC)}.$ Since $0\leq s_j \leq 1$ and for any positive $a,$ $a(x-s_j)_+=\sqrt{a}(\sqrt{a}x-\sqrt{a} s_j)_+,$ we conclude that all network weights can be chosen to be smaller than $2\|f\|_\infty 2^{Kd}.$

{\it (D):} Figure \ref{fig.NWcomp} shows how the neural networks $\wt \phi_K$ and $\wt g$ can be combined into a deep ReLU network with architecture $(2K+3,(d,4d,\ldots,4d,d,1,2^{Kd}+1,1))$ and all network weights bounded in absolute value by $\max(2\|f\|_\infty 2^{Kd}, 2Kd 2^{K\beta p})$ computing the function $\wt f(x_1,\ldots,x_d):=\wt g( 3\sum_{q=1}^d 3^{-q} \wt \phi_K(x_q)).$ Since $3\sum_{q=1}^d 3^{-q} \phi_K(x_q) \in \{s_0,\dots,s_{2^{Kd}}\},$ the interpolation property $\wt g(s_j)=g(s_j)$ implies that $\wt g( 3\sum_{q=1}^d 3^{-q} \phi_K(x_q))=g( 3\sum_{q=1}^d 3^{-q} \phi_K(x_q)).$ Together with (B), we conclude that $$\wt f(x_1, \dots, x_d)=\wt g\Big( 3\sum_{q=1}^d 3^{-q} \wt \phi_K(x_q)\Big)= g\Big( 3\sum_{q=1}^d 3^{-q} \phi_K(x_q)\Big), \  \text{if} \ x_1, \dots, x_d \in \bigcap_{j=1}^K A_{j,r}.$$

As shown in Lemma \ref{lem.f_approx}, $\|f\|_\infty=\|g\|_{L^\infty(\mC)}.$ Since $\wt g$ is a piecewise linear interpolation of $g,$ we also have $\|\wt g\|_{L^\infty([0,1])}\leq \|f\|_\infty.$ As shown in (B), the Lebesgue measure of $(\cap_{j=1}^K A_{j,r})^c$ is bounded by $2^{-K \beta  p}/d.$ Decomposing the integral and using the approximation bound in Lemma \ref{lem.f_approx},
\begin{align*}
	\big\| f - \wt f \,\big\|_p^p 
	&\leq  
	\int_{\forall i:x_i \in \cap_{j=1}^K A_{j,r}} \Big| f(\bx) - g\Big( \sum_{q=1}^d 3^{-q} \phi_K(x_p)\Big)\Big|^p \, d\bx
	+ \int_{\exists i: x_i \notin \cap_{j=1}^K A_{j,r}} 2^p\|f\|_\infty^p \, d\bx \\
	&\leq 2^p Q^p 2^{-\beta Kp} +  2^p\|f\|_\infty^p 2^{-K \beta p} \\
	&\leq 2^p\big(Q + \|f\|_\infty\big)^p 2^{-K \beta p},
\end{align*}
using for the last inequality that $a^p+b^p \leq (a+b)^p$ for all $p\geq 1$ and all $a,b \geq 0.$
\end{proof}

Recall that for a function class with $m^d$ parameters, the expected optimal approximation rate for a $\beta$-smooth function in $d$ dimensions is $m^{-\beta}.$ The previous theorem leads to the rate $2^{-K\beta}$ using of the order of $2^{Kd}$ network parameters. This coincides thus with the expected rate. In contrast to several other constructions, no network sparsity is required to recover the rate. It is unclear whether the construction can be generalized to higher order smoothness or anisotropic smoothness.

The function approximation in Lemma \ref{lem.f_approx} is quite similar to tree-based methods in statistical learning. CART or MARS, for instance, select a partition of the input space by making successive splits along different directions and then fit a piecewise constant (or piecewise linear) function on the selected partition \cite{MR2722294}, Section 9.2. The KA approximation is also piecewise constant and the interior function assigns a unique value to each set in the dyadic partition. Enlarging $K$ refines the partition. The deep ReLU network constructed in the proof of Theorem \ref{thm.main} imitates the KA approximation and also relies on a dyadic partition of the input space. By changing the network parameters in the first layers, the unit cube $[0,1]^d$ can be split in more general subsets and similar function systems as the ones underlying MARS or CART can be generated using deep ReLU networks, see also  \cite{ECKLE2019, 2019arXiv190811140K}.

As typical for neural network constructions that decompose function approximation into a localization and a local approximation step, the deep ReLU network in Theorem \ref{thm.main} only depends on the represented function $f$ via the weights in the last hidden layer. As a consequence, one could use this deep ReLU network construction to initialize stochastic gradient descent. For that, it is natural to sample the weights in the output layer from a given distribution and assign all other network parameters to the corresponding value in the network construction. A comparison with standard network initializations will be addressed in future work. 

The fact that in the proposed network construction only the output layer depends on the represented function matches also with the observation that in deep learning a considerable amount of information about the represented function is decoded in the last layer. This is exploited in pre-training where a trained deep network from a different classification problem is taken and only the output layer is learned by the new dataset, see for instance \cite{Zeiler2014}. The fact that pre-training works shows that deep networks build rather generic function systems in the first layers. For real datasets, the learned parameters in the first hidden layers still exhibit some dependence on the underlying problem and transfer learning updating all weights based on the new data outperforms pre-training \cite{He_2019_ICCV}.

%\cite{KURKOVA1992501, MONTANELLI20201, ISMAILOV2014963}

\section{Related literature}
\label{sec.rel_lit}

The section is intended to provide a brief overview of related approaches. 

\cite{Shen_2020} proposes a similar deep ReLU network construction without making a link to the KA representation or space-filling curves. The similarity between both approaches can be best seen in their Figure 5 or the outline of the proof for Theorem 2.1 in Section 3.2. Indeed, in a first step, the input space is partitioned into smaller hypercubes that are enumerated. The first hidden layers map the input to the index of the hypercube. This localization step is closely related to the action of the interior function used here. The last hidden layers of the deep ReLU network perform a piecewise linear approximation and this is essentially the same as the implementation of the outer function in the modified KA representation in this paper. To ensure good smoothness properties, \cite{Shen_2020} also includes gaps in the indexing, that fulfill a similar role as the gaps in the Cantor set here. \cite{2020arXiv200103040L} combines the approach with local Taylor expansions and achieves optimal approximation rates for functions that are smoother than Lipschitz.

Another direction is to search for activation function with good representation property based on the modified  KA representation in 
\cite{MAIOROV199981}. Their Theorem 4 states that one can find a real analytic, strictly increasing, and sigmoidal activation function $\sigma,$ such that for any continuous function $f:[0,1]^d \to \R$ and any $\eps>0,$ there exist parameters $\bw_{pq} \in \R^d, a_{pq}, b_{pq}, c_q, d_q \in \R,$ satisfying
\begin{align*}
	\sup_{\bx \in [0,1]^d} \bigg |f(\bx) -  \sum_{q=1}^{6d+3} d_q \, \sigma\Big(\sum_{p=1}^{3d} b_{pq}\sigma(\bw_{pq}^\top \bx	+a_{pq})+c_q\Big) \bigg| < \eps.
\end{align*}
This removes the dependence of the outer activation function in the KA representation on the represented function $f.$ The main issue is that despite its smoothness properties, the activation function $\sigma$ is not computable and it is unclear how to transfer the result to popular activation functions such as the ReLU. One step in this direction has been done in the recent work \cite{GULIYEV2018262} proving that one can design computable activation functions with complexity increasing as $\eps \downarrow 0.$ \cite{2020arXiv200612231S} shows that for a neural network with three hidden layers and three explicit and relatively simple but non-differentiable activation functions one can achieve extremely fast approximation rates.

The fact that deep networks can do bit encoding and decoding efficiently has been used previously in \cite{bartlettEtAL2019} to prove (nearly) sharp bounds for the VC dimension of deep ReLU networks and also in \cite{Yarotsky2018, 2019arXiv190609477Y} for a different construction to obtain approximation rates of very deep networks with fixed with. There are, however, several distinctive differences. These works employ bit encoding to compress several function values in one number (see for instance Section 5.2.1 in \cite{Yarotsky2018}), while we apply bit extraction to the input vector. In our approach the bit extraction leads to a localization of the input space and the function values only enters in the last hidden layer. It can be checked that in our construction the weight assignment is continuous, that is, small changes in the represented function will lead to small changes in the network weights. On the contrary, bit encoding in the function values results in discontinuous weight assignment and it is known that this is unavoidable for efficient function approximation based on very deep ReLU networks \cite{Yarotsky2018}.

\bibliographystyle{acm}       % (uses file "plain.bst")
\bibliography{bibDL}           % expects file "refsPart1.bib"

\end{document}